\documentclass[conference]{IEEEtran}
\IEEEoverridecommandlockouts

\usepackage[dvipsnames]{xcolor}
\usepackage{amsmath}
\newcommand*\linkcolours{ForestGreen}

\usepackage{times}
\usepackage{graphicx}
\usepackage{amssymb}
\usepackage{gensymb}
\usepackage{amsmath}
\usepackage{breakurl}

\usepackage{url,hyperref}
\hypersetup{
colorlinks,
linkcolor=\linkcolours,
citecolor=\linkcolours,
filecolor=\linkcolours,
urlcolor=\linkcolours}

\usepackage{algorithm}
\usepackage{algorithmic}

\usepackage[labelfont={bf},font=small]{caption}
\usepackage[none]{hyphenat}
\usepackage{float}
\usepackage{mathtools, cuted}

\usepackage[noadjust, nobreak]{cite}
\usepackage{url}
\usepackage{tabularx}
\usepackage{amsthm}
\usepackage{enumerate}
\usepackage{float}
\usepackage{amsmath}
\usepackage{pifont}

\newcolumntype{Y}{>{\centering\arraybackslash}X}

\usepackage{scrextend}

\usepackage[]{placeins}

\usepackage{placeins}

\usepackage{tikz}

\usepackage[framemethod=tikz]{mdframed}

\usepackage{afterpage}

\usepackage{stfloats}

\usepackage{atbegshi}
\newcommand{\handlethispage}{}
\newcommand{\discardpagesfromhere}{\let\handlethispage\AtBeginShipoutDiscard}
\newcommand{\keeppagesfromhere}{\let\handlethispage\relax}
\AtBeginShipout{\handlethispage}

\usepackage{comment}

\newtheorem{theorem}{Theorem}

\newtheorem{lemma}{Lemma}
\usepackage{environ} 
 
\NewEnviron{EQFONT}{%
    \scalebox{1}{$\BODY$} 
} 
 
\NewEnviron{ALGOFONT}{%
    \scalebox{0.5}{$\BODY$} 
} 
\def\BibTeX{{\rm B\kern-.05em{\sc i\kern-.025em b}\kern-.08em
    T\kern-.1667em\lower.7ex\hbox{E}\kern-.125emX}}
\begin{document}

\title{Query Complexity of $k$-NN based Mode Estimation}

\author{\IEEEauthorblockN{Anirudh Singhal}
\IEEEauthorblockA{\textit{Indian Institute of Technology Bombay} \\
singhalanirudh18@gmail.com}
\and
\IEEEauthorblockN{Subham Pirojiwala}
\IEEEauthorblockA{\textit{Indian Institute of Technology Bombay} \\
pirojiwalasubham@gmail.com}
\and
\IEEEauthorblockN{Nikhil Karamchandani}
\IEEEauthorblockA{\textit{Indian Institute of Technology Bombay} \\
nikhilk@ee.iitb.ac.in}
}

\maketitle

\begin{abstract}
Motivated by the mode estimation problem of an unknown multivariate probability density function, we study the problem of identifying the point with the minimum $k$-th nearest neighbor distance for a given dataset of $n$ points. We study the case where the pairwise distances are apriori unknown, but we have access to an oracle which we can query to get noisy information about the distance between any pair of points. For two natural  oracle models, we design a sequential learning algorithm, based on the idea of confidence intervals, which adaptively decides which queries to send to the oracle and is able to correctly solve the problem with high probability. We derive instance-dependent upper bounds on the query complexity of our proposed scheme and also demonstrate significant improvement over the performance of other baselines via extensive numerical evaluations. 
\end{abstract}


\section{Introduction}
Several problems in machine learning and signal processing are based on computing and processing distances over large datasets, for example high-dimensional clustering and quantization. In several applications, the algorithm might only have access to noisy estimates of the underlying pairwise distances, for example due to the limitations of the measurement process. In this work, we study the problem where we consider a set of points $\mathcal{X} = \{x_1, x_2, \ldots, x_n\} \subset \mathbb{R}^m$ and are interested in finding the point $x^{\star} \in \mathcal{X}$ which has the minimum $k$-nearest neighbor ($k$-NN) distance. This problem finds application in the classical problem of \textit{mode estimation} of an underlying multivariate probability density function $f$ from independent samples \cite{modal_age}, when using a mode estimator of the form $\arg\max_{x \in \mathcal{X} } f_k(x)$, where $f_k$ is the $k$-NN density estimator. This estimator was studied in \cite{ModeEstimationContinous} where it was proven to have several nice properties including consistency and minimax-optimal rates of convergence. 

If all the pairwise distances are available, then the problem of identifying $x^{\star}$ can be solved by sorting them appropriately. In this work, we assume that the algorithm does not know the distances apriori, but has access to an oracle which it can query to get partial/noisy information about the distance between any pair of points. Our goal is to design sequential learning algorithms which adaptively select oracle queries based on the history of responses, and aim to estimate $x^{\star}$ reliably using as few queries as possible. In particular, we consider the following two query models: $(i)$ each query 
 to the oracle returns the distance between a pair of points along a particular dimension. This query model is motivated by recent work in \cite{bagaria2018adaptive} and \cite{adaptiveKNN} which studied random dimension sampling for reducing the computational complexity of problems including $k$-means and nearest neighbor identification; and $(ii)$ when queried with a pair of points, the oracle returns the true pairwise distance corrupted by some additive noise. A similar oracle model was recently employed by \cite{nn_graph} to study the sequential version of the nearest neighbor graph construction problem.

We make the following contributions towards understanding the query complexity of solving the mode estimation problem using a $k$-NN density estimator. \textit{(a)} For both the query models, we design a sequential learning algorithm which at each timestep adaptively chooses node pairs to query based on previous oracle responses and upon stopping, returns the mode $x^{\star}$ with large probability. Furthermore, we prove \textit{instance-dependent} upper bounds on the query complexity of our proposed scheme which explicitly depend on the underlying dataset $\mathcal{X}$. These bounds indicate that the number of oracle queries required depends on the ``hardness" of the underlying instance and on certain easier instances our  scheme can indeed provide significant savings over naive implementations. Our schemes are based on ideas of confidence intervals from online learning and use them to effectively decide the path of exploration across rounds. \textit{(b)} a key sub-routine in our scheme is \textit{Findk-NN} whose goal is to find the $k$-th NN of any point and for a restricted class of datasets and algorithms, we derive fundamental lower bounds on the query complexity of solving this sub-problem. We find that the expressions derived in the lower bound closely resemble those appearing in the query complexity upper bounds, which further validates the efficiency of our proposed scheme. \textit{(c)} We demonstrate the superior performance of our proposed scheme over other baselines by conducting extensive numerical evaluations over the Tiny Imagenet dataset\footnote{ Tiny ImageNet dataset can be dowloaded from \href{https://tiny-imagenet.herokuapp.com/}{here}.\label{footnote:TIMN}}. 

\subsection{Related work}
The problem of mode estimation of an unknown probability density function has a rich history, see for example \cite{modal_age} and references therein, and has seen renewed interest recently due to applications in high-dimensional clustering and regression \cite{mode_clustering, modal_regression}. However, the bulk of the literature has focused on the batch or the non-sequential setting, whereas the goal in this work is to design adaptive algorithms which are robust to noisy/partial information about the pairwise distances. The sequential mode estimation problem for discrete distributions was addressed recently in \cite{adaptiveDiscrete} under somewhat different oracle models. Modern data science applications often involve intensive computations and there have been several recent works which have proposed the use of randomized adaptive algorithms for speeding up a variety of discrete optimization problems such as $k$-NN \cite{adaptiveKNN,nn_graph}, $k$-means \cite{bagaria2018adaptive}, and medoid computation \cite{medoids}. Finally, adaptive decision making and efficient exploration is a central theme of the well-known multi-armed bandit problem \cite{MAB_book}. Ideas such as confidence intervals and change of measure arguments for deriving lower bounds are widely used in the bandit literature, and find application in our work as well as several others mentioned above. Specifically, mapping each pair of points to an arm and treating their pairwise distance as the expected reward, our problem can be cast as a best arm identification problem but with a more complicated score structure instead of the max expected reward, as is standard in the pure exploration framework of multi-armed bandits.
\section{Problem Formulation}

Let $\mathcal{X}=\left\{\mathbf{x}_{1}, \ldots, \mathbf{x}_{n}\right\} \subset \mathbb{R}^{m}$ denote a set of $n$ points. The points in $\mathcal{X}$ are normalized so that $\| \mathbf{x}_i \|_{\infty} \leq 1 / 2$ for all $i$. For every node $\mathbf{x}_i$, we define a set of its neighbours as $ \mathcal{X}_i \triangleq \mathcal{X} \backslash \mathbf{x}_i$ = $\left\{\mathbf{x}^{i}_{1}, \ldots, \mathbf{x}^{i}_{n-1}\right\}$. Let $[\mathbf{x}_i]_p$ be the value of $\mathbf{x}_i$ in the $p^{th}$ dimension, then the distance measure $d^i_j$ is defined as:
\begin{equation}
    {d}_{j}^i \triangleq 
           \frac{1}{m}\left\|\mathbf{x}_i-\mathbf{x}^{i}_j\right\|_{2}^{2} 
           = \frac{1}{m} \sum_{p = 1}^m \left|\left[\mathbf{x}_{i}\right]_{p}-[\mathbf{x}^i_{j}]_{p}\right|^{2}.
\label{eq:d_i_j defined}
\end{equation}

We arrange the elements of $\mathcal{X}_i$, i.e., $\left\{\mathbf{x}^{i}_{1}, \ldots, \mathbf{x}^{i}_{n-1}\right\}$ such that $\mathbf{x}^{i}_{j}$ is the $j^{th}$ nearest neighbour of $\mathbf{x}_i$. Therefore, by construction, we have $d^i_1 \leq d^i_2 \leq \ldots d^i_{n-1}$ and $\mathbf{x}^i_k$ as the $k^{th}$ nearest neighbour ($k$-NN) of $\mathbf{x}_i$.
 
We wish to find the point with the minimum $k$-NN distance, denoted by $\mathbf{x}_{m_k}$, which represents the mode of the dataset $\mathcal{X}$. The work of \cite{ModeEstimationContinous} proved that $\mathbf{x}_{m_k}$ correctly estimates the mode of unimodal continuous distributions, for certain bounds on $k$. Formally, the value of $m_k$ is given by the following equation, where $[n] \triangleq \{1,\ldots,n\}$.
\begin{equation}
m_k \triangleq \underset{i \in [n]}{\operatorname{arg min}} \ d^i_{k}
\label{eq:mk}
\end{equation}


\noindent We study this problem under the following sampling models:

\noindent \textbf{Sampling Model 1} : In this sampling model, we query an oracle with two points $\mathbf{x}_i$ and $\mathbf{x}^i_j$, and a dimension $p$, and it returns the distance between the points along the chosen dimension. The oracle response on querying with $(\mathbf{x}_i,\mathbf{x}^i_j,p)$ is represented by $\mathcal{O}(\mathbf{x}_i,\mathbf{x}^i_j,p)$ and is given as follows:
$$\mathcal{O}(\mathbf{x}_i,\mathbf{x}^i_j,p)=\left|\left[\mathbf{x}_{i}\right]_{p}-[\mathbf{x}^i_{j}]_{p}\right|^{2}.$$

The manner in which $p$, $\mathbf{x}_i$ and $\mathbf{x}^i_j$ are chosen defines the sampling strategy. In the schemes we design, we will employ a random sampling strategy which chooses $p$ uniformly at random from $\{1,\ldots,m\}$ to get a distance estimate between $\mathbf{x}_i$ and $\mathbf{x}^i_j$.
It has been shown by \cite{bagaria2018adaptive} and \cite{adaptiveKNN} that for high-dimensional datasets, the random sampling strategy can achieve computational gains by reducing query complexity.

As $\mathcal{X}$ has been normalized such that $\| \mathbf{x}_i \|_{\infty} \leq 1 / 2$, the outputs of the oracle are bounded in the interval $[0,1]$. Therefore, the oracle outputs are sub-Gaussian random variables with scale parameter $\sigma = \frac{1}{4}$.

\noindent \textbf{Sampling Model 2} : In this sampling model, upon querying an oracle with two points $\mathbf{x}_i$ and $\mathbf{x}^i_j$, it returns a noisy estimate of the distance between the points. The oracle response $\mathcal{O}(\mathbf{x}_i,\mathbf{x}^i_j)$ is given by:
$$\mathcal{O}(\mathbf{x}_i,\mathbf{x}^i_j)=d^i_j + \eta$$

\noindent where $d^i_j$ is as per \eqref{eq:d_i_j defined} and $\eta$ is a zero mean sub-Gaussian random variable with scale parameter $\sigma \leq \frac{1}{4}$.\\

In this paper we aim to find a $\delta$-true sequential algorithm which correctly estimates the mode of $\mathcal{X}$ as defined in \eqref{eq:mk}, with probability at least $1-\delta$, while minimizing the number of oracle queries. The algorithms that we design and the subsequent theoretical analysis to find the query complexity are valid for both the sampling models. Therefore, all the results that we mention will hold for both the sampling models, unless mentioned otherwise. For sampling model 1, we only consider the random sampling strategy.


\subsection{Preliminaries}

Let $d^i_{j,t}$ denote the sample returned by the oracle at iteration $t$ on querying it with $(\mathbf{x}_i,\mathbf{x}^i_j)$\footnote{For sampling model 1, $p$ is chosen uniformly at random from $\{1,\ldots,m\}$}. We define an unbiased estimator of $d^i_j$ after $T^i[j]$ samples as follows:
$$
   \widehat{d}^i\left[j\right] = 
               \frac{1}{T^i[j]} \sum_{t = 1}^{T^i[j]} d^i_{j,t}.
$$

To calculate the confidence bounds for $\widehat{d}^i[j]$, we use a non-asymptotic version of the law of the iterated logarithm \cite{kaufmann2014complexity}, which is stated in the following lemma.

\begin{lemma}
\label{lem:LIL}
(\cite[Lemma 3]{adaptiveKNN}) The following event occurs with probability at least $1-\delta$ for $\delta \in (0,0.05)$:
\begin{multline}
\mathcal{E}_{\beta}:=\{\left|\widehat{d}^{i}[j]-d^{i}_{j}\right| \leq \beta \left( T^{i}[j] \right), \\ \forall i \in [n], \forall j \in [n-1], \forall \ T^i[j] \geq 1\} 
\label{eq:E_beta}
\end{multline}
where the value of $\beta \left( T^{i}[j] \right)$ is given by $\beta \left( T^{i}[j] \right)=\sqrt{\frac{2 \alpha\left(T^i[j], \delta^{'}  \right)}{T^i[j]}}$, $\delta^{'}=\frac{\delta}{n\times(n-1)}$ and  $\alpha(u,\delta^{'}) = \log(1/\delta^{'}) + 3\log\log(1/\delta^{'}) + 1.5\log(1+\log(u))$. For convenience $\beta \left( T^{i}[j] \right)$ is represented by $\beta^i_j$.
\end{lemma}
It is important to note that Lemma \ref{lem:LIL} holds true when the expected value of $\widehat{d}^i[j]$ is $d^i_j$ and it is the sum of $T^i[j]$ independent sub-Gaussian random variables with scale parameter $\sigma \leq \frac{1}{4}$. These conditions hold true for our scheme under both the sampling models.
\section{Algorithm}
We propose a $\delta$-true sequential mode-estimation algorithm called AdapativeModeEstimation (Algorithm \ref{algo:AdaptiveModeEstimation}) which is inspired from the Upper Confidence Bound (UCB) algorithm \cite{bestArmSurvey} popular in the multi-armed bandit (MAB) literature \cite{MAB_book}.  We model a pair $(\mathbf{x}_i,\mathbf{x}^i_k)$ as an arm with the expected reward of the arm being $d^i_k$, where $\mathbf{x}^i_k$ is the $k$-NN of $\mathbf{x}_i$ and $d^i_k$ is the distance between them. In this setting, we have $n$ arms and we wish to find the arm with the least expected reward. This is similar to the best arm identification problem which is widely studied in the multi-armed bandit (MAB) literature and therefore the algorithms we design express this nature.

For sampling model 1, we can devise a naive algorithm where the distances between all the pairs are computed exactly and then the mode (as per \eqref{eq:mk}) is selected. This algorithm would require $O(mn^2)$ queries. We aim to use fewer number of queries by using an adaptive algorithm which works on the principle that some points have much larger $k$-NN distance than the actual mode, and these points can be discarded with high confidence without computing the exact $k$-NN distance, thereby reducing the number of oracle queries. 

\subsection{AdaptiveModeEstimation}
To estimate the mode $m_k$ as per \eqref{eq:mk}, we begin by finding the confidence intervals (CIs) for the $k$-NN distance of each point. Let $U^i_{k}$ and $L^i_{k}$ denote a Upper Confidence Bound (UCB) and a Lower Confidence Bound (LCB) respectively, on the $k$-NN distance of $\mathbf{x}_i$. We find $U^i_{k}$ and $L^i_{k}$ for all $i \in \{1,\cdots, n\}$. At every round, the quantities $l_1$ and $l_2$ are calculated as follows:
 
\begin{equation}
    l_{1} = \underset{i \in\{1, \ldots,n\}}{\arg \min } \ L^i_{k}, l_{2} = \underset{i \in\{1, \ldots,n\}\backslash l_{1}}{\arg \min } \ L^i_{k}.
    \label{eq:b1}
\end{equation}

Then we continuously find tighter CI on the $k$-NN distance, of the point with the minimum LCB on the $k$-NN distance, i.e., $l_1$\footnote{In the original UCB algorithm, the goal is to identify the arm with the highest mean reward. Hence, the arm with the highest UCB is pulled to get tighter CI for it. Here, we wish to identify the point with minimum $k$-NN distance that's why we change the sampling strategy to get tighter CI for the point with minimum LCB on the $k$-NN distance.}. This is done by either calling Find$k$-NN (when the $k$-NN has not yet been identified) or by calling Sample ($\mathbf{x}_{l_1}$,$\mathbf{x}^{l_1}_{b^{l_1}}$) (once the $k$-NN ($\mathbf{x}^{l_1}_{b^{l_1}}$) has been identified). The subroutine Sample($\mathbf{x}_i$,$\mathbf{x}^i_j$) is detailed in Algorithm \ref{algo:Sample}.
Algorithm \ref{algo:AdaptiveModeEstimation} terminates when $U^{l_1}_{k}<L^{l_2}_{k}$, and outputs $\mathbf{x}_{l_1}$. 
A quick summary of the variables can be found in Table \ref{tab:Variable Summary}.

\begin{table}[h]
\begin{center}
 \begin{tabular}{|c c|} 
 \hline
 Variable & Description  \\ [0.5ex] 
 \hline\hline
 
 $\mathcal{X}_i$ & $\mathcal{X} \backslash \mathbf{x}_i$ \\ 
 $\left\{\mathbf{x}^{i}_{1}, \ldots, \mathbf{x}^{i}_{n-1}\right\}$  & Elements of $\mathcal{X}_i$  \\
 \hline
  ${d^i_j}$ & True distance between $\mathbf{x}_i$ and $\mathbf{x}^i_j$  \\
 \hline
 ${b^i}$ & Estimated $k$-NN of $\mathbf{x}_i$ \\
 \hline
 $\widehat{d^i}$ & Array of size $n-1$ $\times$ 1 \\
 $\widehat{d^i}[j]$ & Current distance estimate between $\mathbf{x}_i$ and $\mathbf{x}^i_j$ \\
 \hline
 $T^i$ & Array of size $n-1$ $\times$ 1 \\
 $T^i[j]$ & Number of samples corresponding to $\widehat{d^i}[j]$ \\
 \hline
 $U^i$ and $L^i$  & Arrays of size $n-1$ $\times$ 1 \\
 $U^i[j]$ and $L^i[j]$ & UCB and LCB on $\widehat{d^i}[j]$, respectively \\
 \hline
  $U^i_{k}$ and $L^i_{k}$ & UCB and LCB on the $k$-NN distance of $\mathbf{x}_i$ \\
  \hline
  $\beta^i_j$ & $
  \beta^i_j = \beta(T^i[j])$, $\beta(\cdot)$ is defined in Lemma \ref{lem:LIL} \\
  \hline

\end{tabular}
\end{center}
\caption{Summary of variables}
\label{tab:Variable Summary}
\end{table}

\subsection{Find$k$-NN}
The sub-routine Find$k$-NN (Algorithm \ref{algo:AdaptiveKNN}) finds tighter CI on the $k$-NN distance of the point $\mathbf{x}_i \in \mathcal{X}$, until the $k$-NN has been identified. The proposed algorithm is a variant of the method proposed by \cite{adaptiveKNN} in which they find a set of $k + h$ points (from a dataset of $n$ points) containing the $k$ nearest neighbours of an input point, with high probability. Algorithm \ref{algo:AdaptiveKNN} works by adaptively estimating the distance $d^i_k$ as defined in \eqref{eq:d_i_j defined}.  The arrays $\widehat{d}^i$ and $T^i$ are each of size $n-1$ which store the current estimated distance and number of samples for all the neighbours of $\mathbf{x}_i$, and are are provided as input to Algorithm \ref{algo:AdaptiveKNN}. $U^i$ and $L^i$ are two arrays of size $n-1$, such that $U^i[j]$ and $L^i[j]$ are a UCB and an LCB on $\widehat{d}^i[j]$ respectively.
\begin{algorithm}[t]
\caption{AdaptiveModeEstimation}
\begin{algorithmic}
\STATE \textbf{Input parameters :} $k$, $\mathcal{X}$ and $\delta$
\FOR{$i$ in $\{1,\ldots,n\}$}
    \STATE Initialize $\widehat{d}^i$ and $T^i$ as zero vectors
 
    \STATE Call $\text{Find}k\text{-NN}(\widehat{d}^i,T^i,\mathbf{x}_i)$
    \STATE Update $\widehat{d}^i, T^i, U^i_{k},L^i_{k}, b^i$ and $knnfound^i$
\ENDFOR
\STATE Calculate $l_1$ and $l_2$ as per \eqref{eq:b1}
\WHILE{$U^{l_1}_{k} \geq L^{l_2}_{k}$}
\IF{$knnfound^{l_1} = 0$}
        \STATE Call $\text{Find}k\text{-NN}(\widehat{d}^{l_1},T^{l_1},\mathbf{x}_{l_1})$ 
        \STATE Update $\widehat{d}^{l_1}, T^{l_1}, U^{l_1}_{k},L^{l_1}_{k}, b^{l_1}$ and $knnfound^{l_1}$
\ELSE 
\STATE Call Sample $(\mathbf{x}_{l_1},\mathbf{x}^{l_1}_{b^{l_1}})$   
\ENDIF
    \STATE Update $l_1$ and $l_2$
\ENDWHILE
\RETURN $x_{l_1}$
\end{algorithmic}
\label{algo:AdaptiveModeEstimation}
\end{algorithm}
Let $(\cdot)$ be a permutation of $[n-1]$ such that:
\begin{equation}
\widehat{d^i}[(1)] \leq \widehat{d^i}[(2)] \leq \ldots \widehat{d^i}[(n-1)].
\label{eq:() defination}
\end{equation}

At every call of Algorithm \ref{algo:AdaptiveKNN}, at most three pairs $(\mathbf{x}_i,\mathbf{x}^i_{a_1})$, $(\mathbf{x}_i,\mathbf{x}^i_{a_2})$ and $(\mathbf{x}_i,\mathbf{x}^i_{b})$ are sampled, where $a_1$, $a_2$ and $b$ are given by the following equations.
\begin{equation}
     a_{1} = \underset{j \in\{(1), \ldots,(k-1)\}}{\arg \max } U^i[j]
    \label{eq:q1}
\end{equation}
\begin{equation}
    a_{2} = \underset{j \in\{(k+1), \ldots,(n-1)\}}{\arg \min } L^i[j]
    \label{eq:q2}
\end{equation}
\begin{equation}
    b \triangleq (k)
     \label{eq:c1}   
\end{equation}

\begin{algorithm}[t]
\caption{Find$k$-NN}
\begin{algorithmic}
\STATE \textbf{Input parameters :} $\widehat{d}^{i}$, $T^{i}$ and $x^i$
\FOR{$j$ in $\{1,\ldots,n-1\}$}
    \STATE \textbf{if} $T^{i}[j]=0$ : Call Sample $(\mathbf{x}_i,\mathbf{x}^i_j)$
\ENDFOR
\STATE Calculate $a_1$, $a_2$ and $b$ as per \eqref{eq:q1}, \eqref{eq:q2} and \eqref{eq:c1} respectively

\STATE Call Sample $(\mathbf{x}_i, \mathbf{x}^i_{b})$ and set $knnfound$ $\leftarrow$ 1
\STATE \textbf{if} $U^i[a_1] \geq L^i[b]$ : Call Sample $(\mathbf{x}_i,\mathbf{x}^i_{a_1})$, $knnFound$ $\leftarrow$ 0
\STATE \textbf{if} $U^i[b] \geq L^i[a_2]$ : Call Sample $(\mathbf{x}_i,\mathbf{x}^i_{a_2})$, $knnFound$ $\leftarrow$ 0
\STATE Update $U^i_{k}$ and $L^i_{k}$ as per \eqref{eq: Estimated UB}
\RETURN $\widehat{d^i}$, $T^i$, $U^i_{k}$, $L^i_{k}$, $b$, $knnfound$
\end{algorithmic}
\label{algo:AdaptiveKNN}
\end{algorithm}

The pair $(\mathbf{x}_i,\mathbf{x}^i_b)$ is sampled every time Algorithm \ref{algo:AdaptiveKNN} is called. The pair $(\mathbf{x}_i,\mathbf{x}^i_{a_1})$ is sampled only if $U^i[a_1] > L^i[b]$, because once $U^i[a_1] < L^i[b]$, we can say with high confidence that the $k$-NN $\notin$ $\{(1), \ldots,(k-1)\}$. Similarly, the pair $(\mathbf{x}_i,\mathbf{x}^i_{a_2})$ is sampled only if $U^i[b] > L^i[a_2]$.

Let $\langle \cdot \rangle$ and $\{\cdot\}$ \footnote{The permutations $(\cdot)$, $\langle \cdot \rangle$ and $\{ \cdot \}$ are unique for each $\mathbf{x}_i$. Note that all the three permutations do depend on the node $\mathbf{x}_i$ under consideration. We suppress the dependence in the notation for convenience.} be two permutations of $[n-1]$ such that:
\begin{equation}
U^i[\langle 1 \rangle)] \leq U^i[\langle 2 \rangle] \leq \ldots U^i[\langle n-1 \rangle]
\label{eq: U increasing order}
\end{equation}
\begin{equation}
L^i[\{1\}] \leq L^i[\{2\}] \leq \ldots L^i[\{n-1\}].
\label{eq: L increasing order}
\end{equation}
Algorithm \ref{algo:AdaptiveKNN} calculates and returns a UCB and an LCB on the $k$-NN of $\mathbf{x}_i$ as follows\footnote{Correctness of these bounds are demonstrated in Appendix \ref{appendix:theorem bounds KNN}}:
\begin{equation}
    U^i_{k} = U^i[\langle k\rangle], L^i_{k} = L^i[\{k\}].
    \label{eq: Estimated UB}
\end{equation}

\begin{algorithm}[t]
\caption{Sample}
\begin{algorithmic}
\STATE \textbf{Input parameters :} $\mathbf{x}_i$ and $\mathbf{x}^i_j$
\STATE A query is made to the oracle for the pair $\mathbf{x}_i$ and $\mathbf{x}^i_j$ which returns a sample denoted by $d^i_{j,t}$. For sampling model 1, $p$ is chosen uniformly at random from $\{1,\ldots,m\}$
\STATE $T^i[j] \leftarrow T^i[j]+1$ 
\STATE  $\widehat{{d}^{i}}[j] \leftarrow \frac{T^i[j]-1}{T^i[j]} \widehat{d}^{i}[j]+\frac{1}{T^i[j]} d^i_{j,t} $
\STATE $\beta^i_j \leftarrow \beta \left( T^{i}[j] \right) \footnotemark$
\STATE $U^i[j] \leftarrow \widehat{d_i}[j]+\beta^i_j, L^i[j] \leftarrow \widehat{d_i}[j]-\beta^i_j$ 
\end{algorithmic}
\label{algo:Sample}
\end{algorithm}

\footnotetext{$\beta(\cdot)$ is defined in Lemma \ref{lem:LIL} }

Once we have $U^i[a_1]<L^i[b] \cap L^i[a_2]>L^i[b]$, we set the $knnfound$ flag to $1$ and return $\mathbf{x}^i_b$ as the estimated $k$-NN of $\mathbf{x}_i$. We show that the point returned by Algorithm \ref{algo:AdaptiveModeEstimation} is the correct mode as defined in \eqref{eq:mk}, with probability at least $1-\delta$, in Appendix \ref{appendix:theorem bounds KNN}.
\section{Query Complexity of Algorithm \ref{algo:AdaptiveModeEstimation}}
In this section we find an upper bound on the total number of oracle queries required by AdpativeModeEstimation (Algorithm \ref{algo:AdaptiveModeEstimation}) to find the point $\mathbf{x}_{m_k}$ as defined in \eqref{eq:mk}. To do so, we define the quantity $\Delta^i_{m_k}$ as follows:
\begin{equation}
\Delta^i_{m_k} =  d^i_{k}-d^{m_k}_{k}\\
\label{eq:delta_mk}
\end{equation}

\begin{theorem}
\textit{The number of queries taken by Algorithm \ref{algo:AdaptiveModeEstimation} to identify the point $\mathbf{x}_{m_k}$ as defined in \eqref{eq:mk}, for the dataset $\mathcal{X}$, with probability at least $1-\delta$, is at most $\mathcal{N}_{ME}(\mathcal{X})$.}
\begin{equation}
    \mathcal{N}_{ME}(\mathcal{X})= \sum_{i \in [n]} \mathcal{N}_k(\mathbf{x}_i,\mathcal{X}) + \mathcal{N}_{k}^{\star}(\mathbf{x}_{m_k},\mathcal{X})
\label{eq : UB ME}
\end{equation}
\textit{where $\mathcal{N}_k(\mathbf{x}_i,\mathcal{X})$ is an upper bound on the number of queries required by Find$k$-NN to find the $k$-NN of $\mathbf{x}_i$, defined in \eqref{eq: UB AdaptiveKNN}, and $\mathcal{N}_{k}^{\star}(\mathbf{x}_{m_k},\mathcal{X})$ is defined by the following equation. }
$$
    \mathcal{N}_{k}^{\star}(\mathbf{x}_{m_k}, \mathcal{X})=\widetilde{O}\left(\sum_{i \in [n]\backslash m_k} \left(\Delta^{i}_{m_k}\right)^{-2}\right)
$$\\
\textit{where $\Delta^{i}_{m_k}$ is defined in \eqref{eq:delta_mk}. The logarithmic terms in $n$ and the double logarithmic terms in the gaps are absorbed in $\widetilde{O}$.}
\end{theorem}
\begin{proof}
The number of queries made by Algorithm \ref{algo:AdaptiveModeEstimation} to estimate $\mathbf{x}_{m_k}$ is at most the number of queries required to estimate the $k$-NN of every point and then finding the point with the smallest $k$-NN distance. This is because, for a point $\mathbf{x}_i$, once its $k$-NN ($\mathbf{x}^i_k$) is identified, only the pair $(\mathbf{x}_i,\mathbf{x}^i_k)$ maybe sampled further by Algorithm \ref{algo:AdaptiveModeEstimation}. The number of queries required by Algorithm \ref{algo:AdaptiveModeEstimation} to find the $k$-NN of all the points in $\mathcal{X}$ is at most $\sum_{i \in [n]} \mathcal{N}_k(\mathbf{x}_i,\mathcal{X})$.


Once the $k$-NN of each $\mathbf{x}_i \in \mathcal{X}$ is identified, the problem boils down to finding the point with the lowest $k$-NN distance, which can be modelled as an MAB problem, where we need to find the arm with the least expected reward. As the sampling strategy in Algorithm \ref{algo:AdaptiveModeEstimation} is similar to that of the UCB algorithm (\cite{bestArmSurvey}), the number of oracle queries required after finding the $k$-NN of all the points is at most $\mathcal{N}_{k}^{\star}(\mathbf{x}_{m_k},\mathcal{X})$ (shown in \cite[Section 3.B]{bestArmSurvey}). Therefore, the number of queries required by Algorithm \ref{algo:AdaptiveModeEstimation} is at most the summation of $\sum_{i \in [n]} \mathcal{N}_k(\mathbf{x}_i,\mathcal{X})$ and $\mathcal{N}_{k}^{\star}(\mathbf{x}_{m_k},\mathcal{X})$, as specified in \eqref{eq : UB ME}.

\end{proof}
\section{Bounds on Algorithm \ref{algo:AdaptiveKNN}}
The Find$k$-NN (Algorithm \ref{algo:AdaptiveKNN}) is called repeatedly until the stopping criteria \eqref{eq:sc1} and \eqref{eq:sc2} are met. 
\begin{equation}
    L^i[b]>U^i[a_1]
    \label{eq:sc1}
\end{equation} 
\begin{equation}
 U^i[b]<L^i[a_2] 
    \label{eq:sc2}
\end{equation}

We begin by deriving an upper bound on the query complexity of Find$k$-NN (Algorithm \ref{algo:AdaptiveKNN}). The following quantity for each $i \in [n]$ will be useful in characterizing our bounds on the query complexity.
\begin{equation}
    \Delta^i_{j} = \begin{cases}
              d^i_{k}-d^i_{j} & j<k\\
              \min(d^i_{k}-d^i_{k-1},d^i_{k+1}-d^i_{k}) & j=k \\
              d^i_{j}-d^i_k & j>k
          \end{cases}
\label{eq:delta}
\end{equation}


\begin{theorem}
\textit{The number of queries taken by Algorithm \ref{algo:AdaptiveKNN} to identify the $k$-NN of a point $\mathbf{x}_i$, from the dataset $\mathcal{X} \backslash \mathbf{x}_i$, with probability at least $1-\delta$, is at most $\mathcal{N}_k(\mathbf{x}_i,\mathcal{X})$.}
\begin{equation}
    \mathcal{N}_k(\mathbf{x}_i, \mathcal{X})=\widetilde{O}\left(\sum_{j=1}^{n-1} \left(\Delta^{i}_{j}\right)^{-2}\right)
\label{eq: UB AdaptiveKNN}
\end{equation}\\
\label{theorem : UB adaptiveknn}
\textit{where $\Delta^{i}_{j}$ is defined in \eqref{eq:delta}.}
\end{theorem}

\begin{proof}
For a node $\mathbf{x}_i$, $\mathcal{X}_i \triangleq \mathcal{X} \backslash \mathbf{x}_i$. The elements of $\mathcal{X}_i$ are denoted by $\left\{\mathbf{x}^{i}_{1}, \ldots, \mathbf{x}^{i}_{n-1}\right\}$ and $d^i_1 \leq d^i_2 \leq \ldots d^i_{n-1}$. Here, $a_1$, $a_2$ and $b$ are defined as per \eqref{eq:q1}, \eqref{eq:q2} and \eqref{eq:c1} respectively. Let us define $\mu_1$ and $\mu_2$ as follows:
\begin{equation}
    \mu_1 \triangleq \frac{d^i_{k-1}+d^i_k}{2}, \mu_2 \triangleq \frac{d^i_{k}+d^i_{k+1}}{2}.
 \end{equation}

We define the event that a point $\mathbf{x}^i_j$ is bad as:
\begin{equation}
    \mathcal{E}_{bad}(j)
    = \begin{cases}
              \widehat{d^i}[j] + 3\beta^i_j > \mu_1 & j<k\\
              \widehat{d^i}[j] - 3\beta^i_j < \mu_1 \text{ or } 
              \widehat{d^i}[j] + 3\beta^i_j > \mu_2 & j=k\\
              \widehat{d^i}[j] - 3\beta^i_j < \mu_2 & j>k\\
          \end{cases}
\label{eq:bad events}
\end{equation}
where $\beta^i_j$ is defined in Lemma \ref{lem:LIL}.
Now we state two important lemmas (proved in Appendices \ref{appendix:lemma2} and \ref{appendix:lemma3}) which will help us with the proof.
\begin{lemma}
\textit{If $\mathcal{E}_{\beta}$ \eqref{eq:E_beta} occurs and either of the stopping criteria \eqref{eq:sc1} and \eqref{eq:sc2} is not met, then either $\mathcal{E}_{bad}(a_1)$ or $\mathcal{E}_{bad}(a_2)$ or $\mathcal{E}_{bad}(b)$ occurs. }
\label{lem:Upper bound 1}
\end{lemma}
\begin{lemma}
\textit{Let $z^i_j$ be the smallest integer satisfying $\beta(z^i_j) \leq {\Delta^i_j}/{8}$. If the number of times that the pair $(\mathbf{x}_i,\mathbf{x}^i_j)$ has been queried ($T^i[j]$) is atleast $z^i_j$, then $\mathcal{E}_{bad}(j)$ does not occur. $\Delta^i_j$ and $\beta(\cdot)$ are defined in \eqref{eq:delta} and Lemma \ref{lem:LIL} respectively.}
\label{lem:Upper bound 2}
\end{lemma}
An upper bound on the number of queries is derived by counting all the iterations when the stopping criteria is not met. 
This corresponds to the iterations when either $\mathcal{E}_{bad}(a_1)$ or $\mathcal{E}_{bad}(a_2)$ or $\mathcal{E}_{bad}(b)$ occurs (as per Lemma \ref{lem:Upper bound 1}). 
This further corresponds to the iterations where $T^i[j]$ does not exceed $z^i_j$ for $j \in \{a_1,a_2,b\}$ (as per Lemma \ref{lem:Upper bound 2}). 
Let $t \geq 1$ be the $t^{th}$ iteration of Algorithm \ref{algo:AdaptiveKNN}.
In each iteration, as only $(\mathbf{x}_i, \mathbf{x}^i_{a_1})$, $(\mathbf{x}_i, \mathbf{x}^i_{a_2})$ and $(\mathbf{x}_i, \mathbf{x}^i_{b})$ might be sampled, it follows from Lemma \ref{lem:Upper bound 1} that the total number of queries is bounded by the following quantity (with probability at least $1 - \delta$).
%
\begin{equation}
\begin{split}
    & 3\sum_{t=1}^{\infty} \mathbf{1}\left(\mathcal{E}_{bad}(a_1) \cup \mathcal{E}_{bad}(a_2) \cup \mathcal{E}_{bad}(b) \right) \\
    & \leq 3\sum_{t=1}^{\infty} \sum_{j=1}^{n-1} \mathbf{1} \left( (j=a_1 \cup j=a_2 \cup j=b) \cap \mathcal{E}_{bad}(j) \right) \\
    & \overset{\textit{(i)}}\leq 3\sum_{t=1}^{\infty} \sum_{j=1}^{n-1} \mathbf{1} \left( (j=a_1 \cup j=a_2 \cup j=b) \cap  (T^i[j] \leq z^i_j) \right) \\
    & \overset{\textit{(ii)}}\leq 3\sum_{j=1}^{n-1}z^i_j 
\end{split}
\label{eq:Upper bound inequality}
\end{equation}
%
where \textit{(i)} follows from Lemma \ref{lem:Upper bound 2} and  \textit{(ii)} follows as $T^i[j] \leq z^i_j$ can be true for at most $z^i_j$ iterations. From (\cite[Fact 6]{adaptiveKNN}), $z^i_j \leq \widetilde{O} \left(\left(  \Delta^{i}_{j}\right)^{-2} \right)$, thus proving the theorem.
\end{proof}
We also derive a lower bound on the sample complexity of any algorithm that identifies the $k$-NN of a point, in Appendix \ref{appendix:lower bound on QC}, which is of the same order as the upper bound found in Theorem \ref{theorem : UB adaptiveknn}. 


\section{Experiments}
In this section we evaluate the performance of our proposed scheme on the Tiny ImageNet\footref{footnote:TIMN} dataset, that contains images of dimensions $64\times64\times3$ from 200 different classes. For each trial, we select the dataset by selecting $n$ points at random (without replacement). For sampling model 1, we make a slight modification in our algorithm. We cap the number of queries for any given pair ($\mathbf{x}_i$, $\mathbf{x}^i_j$) to be $m$ and once this limit is reached, evaluate $d^i_j$ exactly by simply querying the dimensions one by one. This is similar to the procedure in \cite{bagaria2018adaptive} and \cite{adaptiveKNN}, and ensures that the worst case query complexity is $O(mn^2)$. In sampling model 2, the oracle response to a query is the true distance corrupted by Gaussian noise with $\sigma = 0.1$. We keep $\delta=0.001$ for all the experiments.

To analyse the accuracy of our algorithm in estimating the mode for a fixed budget of queries, we run our algorithm till the budget is exhausted and then output the point with smallest $k$-NN distance estimate as the mode. We compare the performance of our algorithm with two baselines : \textit{Naive+} and \textit{Random Sampling}. For both the baselines, the query budget is divided equally among all the points. In \textit{Naive+}, for each point, Find$k$-NN is called repeatedly until the individual query budget is met. In \textit{Random Sampling}, for each point, we repeatedly select one of its neighbours uniformly at random, sample the corresponding pair, and update the distance estimate, until the individual query budget is met. For both the baselines, we select the point with the smallest $k$-NN distance estimate as the estimated mode.

\begin{figure}[h]
\centering
\includegraphics[width=0.75\linewidth]{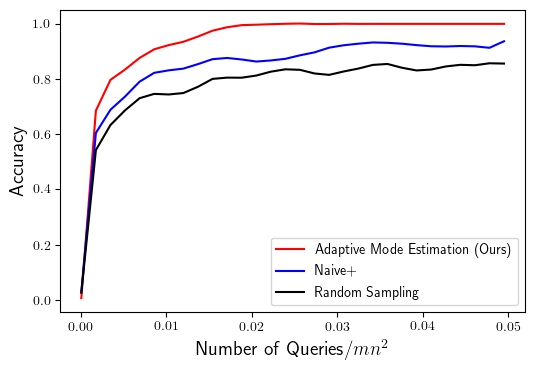}
\caption{ Accuracy vs Number of queries: Sampling model 1 }
\label{fig: NumQueriesAcc1}
\end{figure}

Figure \ref{fig: NumQueriesAcc1} plots the Accuracy vs. Number of Queries curve for our algorithm and the two baselines for sampling model 1. This experiment is performed for $n=100$ and $k=10$ for 200 random trials. As we can see from this figure, our algorithm comprehensively outperforms the two baselines. 
The results for sampling model 2 are  similar and are provided in Appendix \ref{appendix:Additional_Experiments} along with the detailed 
experimental setup and other experimental results.

\bibliographystyle{ieeetr}
\bibliography{bibliography}

\begin{thebibliography}{10}

\bibitem{modal_age}
J.~E. Chac{\'o}n, ``The modal age of statistics,'' in {\em International
  Statistical Review}, vol.~88, pp.~122--141, 2020.

\bibitem{ModeEstimationContinous}
S.~Dasgupta and S.~Kpotufe, ``Optimal rates for k-nn density and mode
  estimation,'' in {\em Advances in Neural Information Processing Systems},
  pp.~2555--2563, 2014.

\bibitem{bagaria2018adaptive}
V.~K. Bagaria, G.~M. Kamath, and D.~N.~C. Tse, ``Adaptive monte-carlo
  optimization,'' {\em arXiv preprint arXiv:1805.08321}, 2018.

\bibitem{adaptiveKNN}
D.~LeJeune, R.~Heckel, and R.~Baraniuk, ``Adaptive estimation for approximate $
  k $-nearest-neighbor computations,'' in {\em International Conference on
  Artificial Intelligence and Statistics}, pp.~3099--3107, 2019.

\bibitem{nn_graph}
B.~Mason, A.~Tripathy, and R.~Nowak, ``Learning nearest neighbor graphs from
  noisy distance samples,'' in {\em Advances in Neural Information Processing
  Systems}, pp.~9586--9596, 2019.

\bibitem{mode_clustering}
Y.-C. Chen, C.~R. Genovese, L.~Wasserman, {\em et~al.}, ``A comprehensive
  approach to mode clustering,'' in {\em Electronic Journal of Statistics},
  vol.~10, pp.~210--241, 2016.

\bibitem{modal_regression}
Y.-C. Chen, C.~R. Genovese, R.~J. Tibshirani, L.~Wasserman, {\em et~al.},
  ``Nonparametric modal regression,'' in {\em The Anals of Statistics},
  vol.~44, pp.~489--514, 2016.

\bibitem{adaptiveDiscrete}
D.~Shah, T.~Choudhury, N.~Karamchandani, and A.~Gopalan, ``Sequential mode
  estimation with oracle queries,'' {\em arXiv preprint arXiv:1911.08197},
  2019.

\bibitem{medoids}
V.~Bagaria, G.~Kamath, V.~Ntranos, M.~Zhang, and D.~Tse, ``Medoids in
  almost-linear time via multi-armed bandits,'' in {\em International
  Conference on Artificial Intelligence and Statistics}, pp.~500--509, PMLR,
  2018.

\bibitem{MAB_book}
T.~Lattimore and C.~Szepesv{\'a}ri, {\em Bandit algorithms}.
\newblock Cambridge University Press, 2020.

\bibitem{kaufmann2014complexity}
E.~Kaufmann, O.~Capp{\'e}, and A.~Garivier, ``On the complexity of best-arm
  identification in multi-armed bandit models,'' in {\em The Journal of Machine
  Learning Research}, vol.~17, pp.~1--42, 2016.

\bibitem{bestArmSurvey}
K.~{Jamieson} and R.~{Nowak}, ``Best-arm identification algorithms for
  multi-armed bandits in the fixed confidence setting,'' in {\em 2014 48th
  Annual Conference on Information Sciences and Systems (CISS)}, pp.~1--6,
  2014.

\bibitem{KL_gaussian}
D.~Johnson and S.~Sinanovic, ``Symmetrizing the kullback-leibler distance,'' in
  {\em IEEE Transactions on Information Theory}, 2001.

\end{thebibliography}


\newif\ifuseRomanappendices
\appendices
\section{Correctness of Algorithm \ref{algo:AdaptiveModeEstimation}}
\label{appendix:theorem bounds KNN}
In this section, we prove that the point returned by Algorithm \ref{algo:AdaptiveModeEstimation} is the one given by \eqref{eq:mk}. To do this, we assume that the event $\mathcal{E}_{\beta}$, which occurs with probability at least $1-\delta$, holds throughout our algorithm.

We know that a UCB and an LCB of the estimated distance between $\mathbf{x}^i$ and $\mathbf{x}^i_j$ are $U^i[j]=\widehat{d}^i[j]+\beta^i_j$ and $L^i[j]=\widehat{d}^i[j]-\beta^i_j$ respectively.
Since $\mathcal{E}_{\beta}$ occurs, it follows from Lemma \ref{lem:LIL} that
\begin{equation}
    L^i[j] \leq d^i_j \leq U^i[j], \forall \ i \in [n] \ \text{and} \ \forall j \in [n-1].
\label{eq: LCB and UCB condition}
\end{equation}

Now we state the following lemma which will help us prove that \eqref{eq: Estimated UB} is a correct CI for $d^i_k$. 

\begin{lemma}
\textit{At most $k-1$ values of $j$ exist for $j \in [n-1]$ such that $d^i_j < L^i[\{k\}]$.}
\label{lem:LCB_k}
\end{lemma}
\begin{proof}
We prove Lemma \ref{lem:LCB_k} by contradiction. Let us assume there exists a set of $k-1+l$ points for $l\geq 1$, represented by $\{p_1,p_2,\ldots,p_{k-1+l}\}$, such that:
$$d^i_{p_j}<L^i[\{k\}], \forall j \in [k-1+l].$$

Since $\mathcal{E}_{\beta}$ occurs, it follows from \eqref{eq: LCB and UCB condition} that $L^i[j] \leq d^i_j$. Hence, we have:
$$L^i[p_j]\leq d^i_{p_j} < L^i[\{k\}], \forall j \in [k-1+l].$$

This implies that atleast $k-1+l$ values of $j$ exist such that $L^i[j] < L^i[\{k\}]$. However, from the definition of the permutation $\{\cdot\}$  in \eqref{eq: L increasing order}, we know that exactly $k-1$ such values of $j$ exist. 
Since we assumed $l \geq 1$, we have a contradiction. 
\end{proof}

Using Lemma \ref{lem:LCB_k}, we can easily prove the following theorem which shows that \eqref{eq: Estimated UB} is a correct CI for $d^i_k$. 
\begin{theorem}
\label{theorem:bounds on KNN}
\textit{If the event $\mathcal{E}_\beta$ occurs, then the $k$-NN distance of $\mathbf{x}_i$, denoted by $d^i_k$, is bounded as follows:
$$
    L^i[\{k\}] \leq d^i_k \leq U^i[\langle k \rangle], \forall i \in [n], \forall k \in [n-1]
$$
where the definitions of $\langle\cdot\rangle$ and $\{\cdot\}$ are as per \eqref{eq: U increasing order} and \eqref{eq: L increasing order} respectively.}\\
\end{theorem}
\begin{proof}

Lemma \ref{lem:LCB_k} immediately implies that $L^i[\{k\}] \leq d^i_k$. Using similar arguments, we can also show that $d^i_k \leq U^i[\langle k\rangle]$, thus proving the theorem. 
\end{proof}
We now show that when the event $\mathcal{E}_\beta$ occurs, the value returned by Algorithm \ref{algo:AdaptiveModeEstimation}, i.e. $\mathbf{x}_{l_1}$, is same as $\mathbf{x}_{m_k}$ defined in \eqref{eq:mk}. When the algorithm terminates, we have the following.
\begin{equation}
    U^{l_1}[\langle k \rangle] = U^{l_1}_k < L^{l_2}_k = L^{l_2}[\{k\}]
    \label{eq:algo_terminates}
\end{equation}
where $l_1$ and $l_2$ are defined in \eqref{eq:b1}.
From Theorem \ref{theorem:bounds on KNN}, we have the following two equations:
\begin{equation}
    L^{l_2}[\{k\}] \leq d^i_k, \forall i \in [n]\backslash l_1
    \label{eq:b2_lcb}
\end{equation}
\begin{equation}
    d^{l_1}_k \leq U^{l_1}[\langle k \rangle].
    \label{eq:b1_ucb}
\end{equation}
Plugging in \eqref{eq:b2_lcb} and \eqref{eq:b1_ucb} in \eqref{eq:algo_terminates}, we get the following result.
$$d^{l_1}_{k}\leq U^{l_1}[\langle k \rangle]<L^{l_2}[\{k\}]\leq d^i_k, \forall i \in [n-1]\backslash l_1$$
$$\implies d^{l_1}_{k} < d^i_k , \forall i \in [n]\backslash l_1$$

Therefore, when the event $\mathcal{E}_\beta$ occurs, we have $l_1=m_k$. Hence, Algorithm \ref{algo:AdaptiveModeEstimation} provides the correct result with probability atleast $1-\delta$.
\section{Proof of lemma  \ref{lem:Upper bound 1}}
\label{appendix:lemma2}

Let us define $\mathcal{E}_{good}(j)$ as the complement of $\mathcal{E}_{bad}(j)$ defined in \eqref{eq:bad events}. 

Let $\widehat{S}_{near}=\{(1),(2),\ldots,(k-1)\}$ and $\widehat{S}_{far}=\{(k+1),(k+2),\ldots,(n-1)\}$, where $(\cdot)$ is a permutation of $[n-1]$ as per \eqref{eq:() defination}.

We prove Lemma \ref{lem:Upper bound 1} by proving that occurrence of $\mathcal{E}_{good}(j)$ $\forall {j \in \{a_1, a_2, b\}}$ implies that both the stopping criteria \eqref{eq:sc1} and \eqref{eq:sc2} are met. Assuming $\mathcal{E}_{good}(j)$ occurs $\forall {j \in \{a_1, a_2, b\}}$, we first prove that, for all possible values of $b$ and $a_1$, either the stopping criterion \eqref{eq:sc1} is met or there is a contradiction. We assume that the event $\mathcal{E}_{\beta}$, which occurs with probability at least $1-\delta$, holds throughout our algorithm.
\begin{enumerate}[i.]
    \item \textbf{Case 1}: \{$a_1\geq k$, $b\leq k-1$\} \text{OR} \{$a_1\geq k+1$, $b= k$\}:\\
        $\mathcal{E}_{good}(a_1)$ and $\mathcal{E}_{good}(b)$ implies that $\widehat{d^i}[a_1] - 3\beta^i_{a_1} \geq \widehat{d^i}[b] + 3\beta^i_{b}$, which further implies that $\widehat{d^i}[a_1] > \widehat{d^i}[b]$ (as $\beta^i_j>0$ for all possible values of $j$). This is a contradiction as, by definition, $a_1 \in \widehat{S}_{near}$ and $b=(k)$, implying $\widehat{d^i}[a_1] < \widehat{d^i}[b]$. Hence, this case is not valid by definition.
        
     \item \textbf{Case 2}: \{$a_1\leq k-1$, $b= k$\} \text{OR}  \{$a_1\leq k$, $b\geq k+1$\}:\\
         $\mathcal{E}_{good}(a_1)$ and $\mathcal{E}_{good}(b)$ implies that $\widehat{d^i}[b] - 3\beta^i_{b} \geq \widehat{d^i}[a_1] + 3\beta^i_{a_1}$, which further implies that $\widehat{d^i}[b] - \beta^i_{b} > \widehat{d^i}[a_1] + \beta^i_{a_1}$. Therefore, $L^i[b]>U^i[a_1]$ and the stopping criterion \eqref{eq:sc1} is met.

    \item \textbf{Case 3}: \{$a_1\leq k-1$, $b\leq k-1$\}:\\
        $\mathcal{E}_{good}(a_1)$ implies that $\widehat{d^i}[a_1] + 3\beta^i_{a_1} \leq \mu_1$, which further implies that $\widehat{d^i}[a_1] + \beta^i_{a_1} \leq \mu_1$. Therefore, by definition of $a_1$ (point with the highest UCB in $\widehat{S}_{near}$) and $\mathcal{E}_{good}(b)$, we have $\widehat{d^i}[j] + \beta^i_{j} < \mu_1$ for all $j \in \widehat{S}_{near} \cup \{b\}$. Since $\mathcal{E}_\beta$ occurs, we have $d^i_j<\widehat{d}_i[j]+\beta^i_j$. This, coupled with the fact that $\mu_1 = \frac{d^i_{k-1}+d^i_k}{2} \leq d^i_k$ gives us $d^i_j<d^i_k$ $\forall$ $j \in \widehat{S}_{near} \cup \{b\}$. This is a contradiction as this implies $d^i_j<d^i_k$ holds for $k$ points, but by definition of $d^i_k$, it can only hold true for $k-1$ points. Hence, this case is not valid by definition.
    
    \item \textbf{Case 4}: \{$a_1\geq k+1$, $b\geq k+1$\}:\\
         \textit{Case} $4(i)$: \{$a_2\leq k$\}:\\
        We can arrive at a contradiction using similar arguments as that of \textit{Case 1}. Hence, this case is not valid by definition.

         \textit{Case} $4(ii)$: \{$a_2\geq k+1$\}:\\
        We can arrive at a contradiction using similar arguments as that of \textit{Case 3}. Hence, this case is not valid by definition.
\end{enumerate}

Assuming $\mathcal{E}_{good}(j)$ occurs $\forall {j \in \{a_1, a_2, b\}}$, we can use similar arguments to prove that, for all possible values of $b$ and $a_2$, either the stopping criterion \eqref{eq:sc2} is met or there is a contradiction. Hence, for all possible values of $b$, $a_1$ and $a_2$, either there is a contradiction, or both the stopping criteria \eqref{eq:sc1} and \eqref{eq:sc2} are met. Hence, the result is proved.
\section{Proof of lemma  \ref{lem:Upper bound 2}}
\label{appendix:lemma3}
Let us take the case when $j < k$. Given $\beta^i_j \leq {\Delta^i_j}/{8}$, 
 \begin{equation*}
     \begin{split}
          \widehat{d}^i[j] + 3\beta^i_j & \overset{\textit{(i)}} \leq {d}^i_j + 4\beta^i_j \\
         &  \leq {d}^i_j + \frac{\Delta^i_j}{2} \\
         & \overset{\textit{(ii)}} \leq \frac{{d}^i_k + {d}^i_{k-1}}{2} \\
         & = \mu_1
     \end{split}
 \end{equation*}
where inequality $(i)$ follows from the assumption that $\mathcal{E}_{\beta}$ occurs and inequality $(ii)$ follows from the definition of $\Delta^i_j$. Hence, $\mathcal{E}_{bad}(j)$ does not occur for $j<k$.
For $j>k$, the proof is similar.\\
\\
For $j = k$, 
 \begin{equation*}
     \begin{split}
          \widehat{d}^i[k] + 3\beta^i_k  & \overset{\textit{(i)}} \leq {d}^i_k + 4\beta^i_k \\
         & \leq {d}^i_k + \frac{\Delta^i_k}{2} \\
         & \overset{\textit{(ii)}} \leq \frac{{d}^i_{k+1} + {d}^i_{k}}{2} \\
         & = \mu_2
     \end{split}
 \end{equation*}
\hspace{4cm}\text{AND}
  \begin{equation*}
     \begin{split}
          \widehat{d}^i[k] - 3\beta^i_k & \overset{\textit{(iii)}} \geq {d}^i_k - 4\beta^i_k \\
         & \geq {d}^i_k - \frac{\Delta^i_k}{2} \\
         & \overset{\textit{(iv)}} \geq \frac{{d}^i_k + {d}^i_{k-1}}{2} \\
         & = \mu_1
     \end{split}
 \end{equation*}
 where inequalities $(i)$ and $(iii)$ follow from the assumption that $\mathcal{E}_{\beta}$ occurs and
  inequalities $(ii)$ and $(iv)$ follow from the fact that $\Delta^i_k = \min(d^i_{k}-d^i_{k-1},d^i_{k+1}-d^i_{k})$. Hence, $\mathcal{E}_{bad}(j)$ does not occur for $j=k$ either, thus proving the result.
  
\section{Lower Bound on Query Complexity of Algorithm \ref{algo:AdaptiveKNN}}
\label{appendix:lower bound on QC}

In this section, we find a lower bound on the number of queries of all the active algorithms which estimate the $k$-NN distance of a point. We do this by mapping the problem to a multi-armed bandit (MAB) setting where, for a reference point $\mathbf{x}_i$, we have $n-1$ arms, one corresponding to each of the other points. The expected reward of an arm in this mapping is the distance between the reference point and the corresponding neighbour. In this setting, finding the $k$-NN corresponds to finding the arm with $k^{th}$ smallest expected reward. This can be seen as a variant of the widely studied pure exploration problem in the MAB setting \cite{bestArmSurvey} where the goal is to identify the arm with the highest expected reward. We make the following assumptions on the dataset, the oracle and the algorithms which interact with it, to ensure that the lower bound is well defined.
\begin{itemize}
    \item \textbf{Sampling Model 1} : Similar to lower bound argument in \cite{adaptiveKNN}, for this sampling model we will restrict our attention to datasets where $[\mathbf{x}_i]_p \in \{-\frac{1}{2},\frac{1}{2}\}$ $\forall i \in [n]$ and $\forall p \in [m]$, and algorithms which only interact with the dataset at any time by accessing the distance between a pair of points across a randomly chosen dimension. Thus, the oracle response for any pair $(\mathbf{x}_i,\mathbf{x}^i_j)$ has a Bernoulli distribution with mean $d^i_j$.
    \item \textbf{Sampling Model 2} : For this sampling model, we assume that oracle response for a pair $(\mathbf{x}_i,\mathbf{x}^i_j)$ has a Gaussian distribution with mean $d^i_j$ and variance $\sigma^2$.
\end{itemize}

Next, we introduce some notation from the bandit literature \cite{MAB_book} to help us find the lower bound. Let $\nu=\left(\nu_{1}, \ldots, \nu_{n-1}\right)$ be a set of $n-1$ probability distributions which represents an MAB instance with $\mathbb{E}[\nu_1]<\mathbb{E}[\nu_2]<\ldots<\mathbb{E}[\nu_{n-1}]$. Under our assumptions for sampling model 1, $\nu_j$ has a Bernoulli distribution with mean $d^i_j$, and for sampling model 2, $\nu_j$ has a Gaussian distribution with mean $d^i_j$ and variance $\sigma^2$ for all $j\in [n-1]$. Let $\mathcal{A}$ be an active algorithm\footnote{\label{footnote:RS} For Sampling Model 1, we only consider the algorithms which follow the random sampling strategy.} which identifies the arm with the $k^{th}$ smallest reward ($\nu_k$), with probability at least $1-\delta$. The algorithm selects an arm $i_t$ at time instant $t$ and gets the reward $X_t$ from the distribution $\nu_{i_t}$. The arm selected by the algorithm at a time instant $t$ is based only on past arm selections and rewards. Let $\mathcal{F}_t$ be the $\sigma$-algebra generated from $i_1,X_1,\ldots,i_t,X_t$. The algorithm terminates when some stopping rule $\xi$ is satisfied. Let $\tau$ denote the stopping time based on the stopping rule $\xi$. We assume that $P\left[\tau<\infty\right]=1$. 

\begin{theorem}
\label{theorem:Lower Bound on query complexity}
\textit{Let $\nu=\left(\nu_{1}, \ldots, \nu_{n-1}\right)$ be an MAB instance such that $\mathbb{E}[\nu_1]<\mathbb{E}[\nu_2]<\ldots<\mathbb{E}[\nu_{n-1}]$ and $\mathcal{A}$ be an algorithm\footref{footnote:RS} that identifies the arm with the $k^{th}$ smallest expected reward ($\nu_k$), with probability at least $1-\delta$. Let $\tau$ represent the number of samples the algorithm requires to terminate. Then for all $\delta \in (0,0.15]$ and for all such algorithms the following condition is true.}
\begin{multline*}
\hspace*{-0.5cm} \mathbb{E}_{\nu}[\tau] \geq \log \left(\frac{1}{2.4 \delta}\right)\left[\sum_{a \in [n-1]\backslash k} \frac{1}{\mathrm{KL}\left(\nu_{a}, \nu_{k}\right)}+\frac{1}{\mathrm{KL}\left(\nu_{k}, \nu_{k^{*}}\right)} \right] 
\end{multline*}
\textit{where $k^{*}=\underset{a \in\{k-1,k+1\}}{\arg \min } \left|\mathbb{E}[\nu_a]-\mathbb{E}[\nu_k]\right|$, and $\mathrm{KL}(p,q)$ is the Kullback-Leibler divergence between two probability distributions $p$ and $q$.}\\
\end{theorem}

\begin{proof}
Let $\mathrm{kl}(p, q):=p \log \frac{p}{q}+(1-p) \log \frac{1-p}{1-q}$ be the KL divergence between two Bernoulli distributions with mean $p$ and $q$ respectively. Let $N_i[\tau]$ denote the number of times algorithm $\mathcal{A}$ takes reward from $\nu_i$. Next, we state a lemma from \cite{kaufmann2014complexity} to help us in our proof of Theorem \ref{theorem:Lower Bound on query complexity}.

\begin{lemma}
\label{lemma:kaufmann}
\textit{(\cite[Lemma 1]{kaufmann2014complexity}) Let $\nu=\left(\nu_{1}, \ldots, \nu_{n}\right)$ and $\nu^{\prime}=\left(\nu^{\prime}_{1}, \ldots, \nu^{\prime}_{n}\right)$ be two MAB instances. Then for any event $\mathcal{E} \in \mathcal{F}_{\tau}$ with $0<P_{\nu}(\mathcal{E})<1$, the following is true:}
\begin{equation*}
\sum_{i=1}^{n} \mathbb{E}_{\nu}\left[N_{i}(\tau)\right] \operatorname{KL}\left(\nu_{i}, \nu_{i}^{\prime}\right) \geq \operatorname{kl}\left(\mathrm{P}_{\nu}[\mathcal{E}], \mathrm{P}_{\nu^{\prime}}[\mathcal{E}]\right).
\end{equation*}
\end{lemma}

Next, we make an assumption (which has also been made in \cite{kaufmann2014complexity}) that for any two probability distributions $p$ and $q$, and for any $\alpha$, the following is true.
\begin{itemize}
    \item There exists a probability distribution $q_1$ such that $\mathrm{KL}(p,q)<\mathrm{KL}(p,q_1)<\mathrm{KL}(p,q)+\alpha$ and $\mathbb{E}_{X \sim q_{1}}[X]>\mathbb{E}_{X \sim q}[X]$.
    \item There exists a probability distribution $q_2$ such that $\mathrm{KL}(p,q)<\mathrm{KL}(p,q_2)<\mathrm{KL}(p,q)+\alpha$ and $\mathbb{E}_{X \sim q_{2}}[X]<\mathbb{E}_{X \sim q}[X]$.
\end{itemize}

Now we begin with our proof of Theorem \ref{theorem:Lower Bound on query complexity}. For a MAB instance $\nu=\left(\nu_{1}, \ldots, \nu_{n}\right)$ with $\mathbb{E}[\nu_{1}]<\mathbb{E}[\nu_{2}]<\ldots<\mathbb{E}[\nu_{n}]$, we select an alternative MAB instance $\nu^{\prime}=\left(\nu_{1},\ldots,\nu_{a-1},\nu^{\prime}_{a},\nu_{a+1},\ldots,\nu_{n}\right)$, where $\nu^{\prime}_{a}$ is defined such that following is true for a fixed $\alpha$.\ 
\begin{itemize}
    \item $\mathrm{KL}(\nu_a,\nu_k)<\mathrm{KL}(\nu_a,\nu^{\prime}_a)<\mathrm{KL}(\nu_a,\nu_k)+\alpha$ and $\mathbb{E}[\nu^{\prime}_a]>\mathbb{E}[\nu_k]$ for $a<k$.
    \item $\mathrm{KL}(\nu_a,\nu_k)<\mathrm{KL}(\nu_a,\nu^{\prime}_a)<\mathrm{KL}(\nu_a,\nu_k)+\alpha$ and $\mathbb{E}[\nu^{\prime}_a]<\mathbb{E}[\nu_k]$ for $a>k$.
    \item $\mathrm{KL}(\nu_a,\nu_{k^{*}})<\mathrm{KL}(\nu_a,\nu^{\prime}_a)<\mathrm{KL}(\nu_a,\nu_{k^*})+\alpha$ for $a=k$, where $k^{*}={\arg \min } \left|\mathbb{E}[\nu_j]-\mathbb{E}[\nu_k]\right|$ for $j \in\{k-1,k+1\}$, and   $\mathbb{E}[\nu^{\prime}_a]<\mathbb{E}[\nu_{k-1}]$ if $k^{*}=k-1$, otherwise  $\mathbb{E}[\nu^{\prime}_a]>\mathbb{E}[\nu_{k+1}]$.
\end{itemize}
This construction of $\nu^{\prime}$ ensures that the arm with the $k^{th}$ smallest reward in $\nu$ and $\nu^{\prime}$ are different.

We know that $\xi$ is the stopping criteria of algorithm $\mathcal{A}$ for the MAB instance $\nu$. As $\tau$ is the stopping time, we have $\xi \in \mathcal{F}_{\tau}$. By definition, we know that $\mathrm{P}_{\nu}[\xi]\geq 1-\delta$ and $\mathrm{P}_{\nu^{\prime}}[\xi]\leq \delta$. This is because, as $\xi$ is the stopping criteria of $\mathcal{A}$ for MAB $\nu$, it will succeed in it with probability at least $1-\delta$, and fail with any other MAB instance with different optimal arm with probability at most $\delta$. Also, it is easy to show that $\mathrm{KL}(\mathrm{P}_{\nu}[\xi],\mathrm{P}_{\nu^{\prime}}[\xi]) = \mathrm{kl}(1-\delta,\delta) \geq \log (\frac{1}{2.4\delta})$, where the last inequality holds for $\delta \leq 0.15$. Plugging these values in Lemma \ref{lemma:kaufmann}, we get the following equation.
$$\mathrm{KL}\left(\nu_{a}, \nu_{a}^{\prime}\right) \mathbb{E}_{\nu}\left[N_{a}(\tau)\right] \geq \mathrm{kl}(1-\delta, \delta) \geq \log (\frac{1}{2.4 \delta})$$
This is because our construction of $\nu^{\prime}$ is such that $\nu_i=\nu^{\prime}_i$ for $i\neq a$. This implies that $\mathrm{KL}(\nu_i,\nu_i^{\prime})=0$ for $i\neq a$. From our construction of $\nu^{\prime}_a$ and the above inequality, we have the following.

$$
    \mathbb{E}_{\nu}\left[N_{a}(\tau)\right] \geq \begin{cases}
               \log (\frac{1}{2.4 \delta})/\left(\mathrm{KL}(\nu_a,\nu_k)+\alpha \right) & a\neq k\\
               \log (\frac{1}{2.4 \delta})/\left(\mathrm{KL}(\nu_a,\nu_{k^{*}})+\alpha \right) & a= k
           \end{cases}
\label{eq:delta3}
$$
Moreover,
$$\text{ } \mathbb{E}_{\nu}\left[\tau\right] = \sum_{a=1}^{n} \mathbb{E}_{\nu}\left[N_{a}(\tau)\right].
$$
This implies that 
$$
\hspace*{-0.1cm} \mathbb{E}_{\nu}[\tau] \geq \log \left(\frac{1}{2.4 \delta}\right)\left[\sum_{a \in [n]\backslash k} \frac{1}{\mathrm{KL}\left(\nu_{a}, \nu_{k}\right)+\alpha}+\frac{1}{\mathrm{KL}\left(\nu_{k}, \nu_{k^{*}}\right)+\alpha} \right].
\label{eq:Lower bound on QC}
$$
As $\alpha \rightarrow 0$, the above equation gives the result in Theorem \ref{theorem:Lower Bound on query complexity}, thereby completing our proof.
\end{proof}

Using Theorem \ref{theorem:Lower Bound on query complexity}, we are going to find a lower bound on the query complexity for both the sampling models. For sampling model 1, we have assumed that $\nu_j$ has a Bernoulli distribution with $\mathbb{E}[\nu_a]=d^i_j$. Therefore, $\mathrm{KL}\left(\nu_{a}, \nu_{k}\right)$ if given by:
$$\mathrm{KL}\left(\nu_{a}, \nu_{k}\right) = d^i_a \log \left(\frac{d^i_a}{d^i_k}\right) + (1-d^i_a)\log \left(\frac{1-d^i_a}{1-d^i_k}\right).$$
Using the inequality $\log(x) \leq x-1$ we have that:
\begin{equation}
    \label{eq: SM1 LB}
    \mathrm{KL}\left(\nu_{a}, \nu_{k}\right) \leq \frac{(d^i_a-d^i_k)^2}{d^i_k(1-d^i_k)} = \frac{(\Delta^i_a)^2}{d^i_k(1-d^i_k)}.
\end{equation}

In sampling model 2, we have assumed that $\nu_j$ has a Gaussian distribution with mean $d^i_j$ and variance $\sigma^2$. Therefore, from \cite[Table 1]{KL_gaussian}, we have that:
\begin{equation}
    \label{eq: SM2 LB}
    \mathrm{KL}\left(\nu_{a}, \nu_{k}\right) = \frac{(d^i_a-d^i_k)^2}{2\sigma^2} = \frac{(\Delta^i_a)^2}{2\sigma^2}.
\end{equation}

If we plug in \eqref{eq: SM1 LB} and \eqref{eq: SM2 LB} in the result of Theorem \ref{theorem:Lower Bound on query complexity}, we get that $\mathbb{E}_{\nu}[\tau] \geq \mathcal{N}_{low}(\mathbf{x}_i,\mathcal{X})$, where $\mathcal{N}_{low}(\mathbf{x}_i,\mathcal{X})$ is given by the following equation.

$$
    \mathcal{N}_{low}(\mathbf{x}_i, \mathcal{X})=\widetilde{O}\left(\sum_{j=1}^{n-1} \left(\Delta^{i}_{j}\right)^{-2}\right)
$$

This holds for both the sampling models, under our assumptions. As $\mathcal{N}_{low}(\mathbf{x}_i,\mathcal{X})$ and $\mathcal{N}(\mathbf{x}_i,\mathcal{X})$ are of the same order, we can say that Algorithm \ref{algo:AdaptiveKNN} is optimal.
\section{Experimental Setup and Additional Experiments}
\label{appendix:Additional_Experiments}
In this section we provide additional details about the experimental setup. All the plots that we have reported have been passed through a Savitzky-Golay filter to make them smoother. For sampling model 1, as the naive method to find the mode requires $mn^2$ queries, we report the number of queries for this model after dividing them by $mn^2$, to understand the reduction in queries achieved. For uniformity, we normalize the number of queries for sampling model 2 by $mn^2$ as well.

The concentration bound used in the experiments is $\beta(u) = \sqrt{(C_{\beta} \log(1+(1+\log(u))n/\delta))/u}$. This concentration bound is same as the one used in \cite{adaptiveKNN} during experiments. We choose the appropriate value of $C_\beta$ by varying it and studying its affect on the accuracy and the the number of queries. For each value of $C_\beta$, we perform 25 random trials and the average accuracy and the average number of queries are plotted. We perform this experiment for $n=100$ and $k=10$. The results are plotted in Figure \ref{fig: C_beta_var_1} (Sampling model 1) and \ref{fig: C_beta_var_2} (Sampling model 2). A higher $C_{\beta}$ implies a larger confidence interval which results in  better accuracy, but also leads to an increase in the query complexity. This plot helps us to choose the value of $C_{\beta}$, for a given requirement on accuracy. From these figures, we have selected the value of $C_\beta$ to be $0.03$ for sampling model 1 and $0.01$ for sampling model 2, so as to keep the average accuracy = $1$ for all the subsequent experiments.

\begin{figure}[!htb]
\centering
\includegraphics[width=\linewidth]{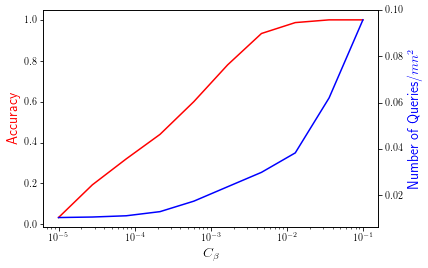}
\caption{ Varying $C_{\beta}$: Sampling model 1  }
\label{fig: C_beta_var_1}
\end{figure}

\begin{figure}[!htb]
\centering
\includegraphics[width=\linewidth]{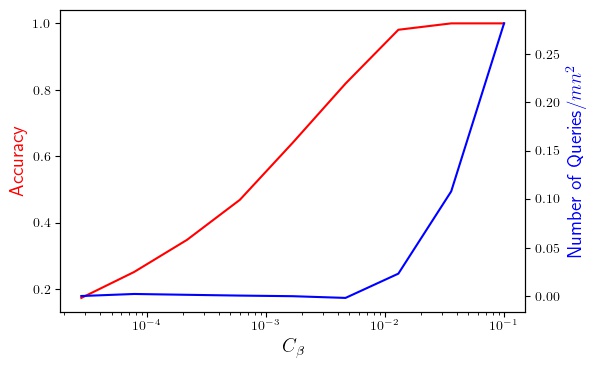}
\caption{ Varying $C_{\beta}$: Sampling model 2 }
\label{fig: C_beta_var_2}
\end{figure}

\subsection{Experiment 1: Accuracy vs Number of queries}
\label{C_beta variation}
Similar to Figure \ref{fig: NumQueriesAcc1}, we also compare the performance of our algorithm, for a fixed budget of queries, with \textit{Random Sampling} and \textit{Naive+}, for sampling model 2. We perform 200 random trials and the average accuracy for different number of queries are plotted for $n=100$ and $k=10$. 

Figure \ref{fig: NumQueriesAcc2} contains the plot of Accuracy vs. Number of Queries for Algorithm \ref{algo:AdaptiveKNN} and the two baselines, for sampling model 2. Similar to sampling model 1, Algorithm \ref{algo:AdaptiveKNN} comprehensively outperforms the two baselines. It has better accuracy for any given number of queries, and it also requires fewer number of queries to achieve accuracy = 1.

\begin{figure}[h]
\centering
\includegraphics[width=\linewidth]{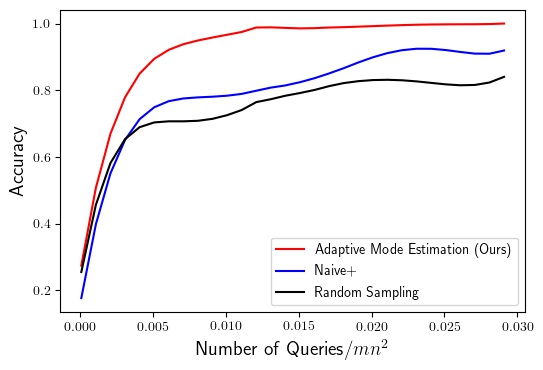}
\caption{ Accuracy vs Number of queries: Sampling model 2}
\label{fig: NumQueriesAcc2}
\end{figure}

\subsection{Experiment 2: Varying n }
\label{n variation}
Here, we study the effect of increasing $n$ on the number of queries, for average accuracy = $1$, for different values of $k$. As our upper bound \eqref{eq : UB ME} has $O(n^2)$ summation terms, normalizing the number of queries by $mn^2$ helps to understand the effect of increasing $n$ better. For every value of $n$, we perform 20 random trials and the average number of queries are plotted. For a given value of $n$, we choose 3 different values of  $k$, i.e., $k=\frac{n}{5}$, $k=\frac{n}{10}$ and $k=\frac{n}{20}$.

The results are plotted in Figure \ref{fig: N_variation_1} (Sampling model 1) and \ref{fig: N_variation_2} (Sampling model 2). From the plots, we can conclude that the rate of increase in the number of queries w.r.t $n$ is less than that of $n^2$. Hence, gain in terms of the number of queries is more significant for larger datasets. Also, we can see that for the Tiny Imagenet dataset, the number of queries increase as the value of $k$ increases. This is as expected since the task of estimating the $k$-NN distance becomes harder as more nodes from the dataset are included.

\begin{figure}[!htb]
\centering
\includegraphics[width=\linewidth]{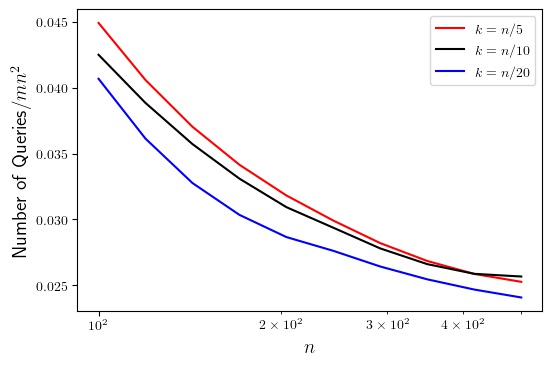}
\caption{ Varying $n$: Sampling model 1  }
\label{fig: N_variation_1}
\end{figure}

\begin{figure}[!htb]
\centering
\includegraphics[width=\linewidth]{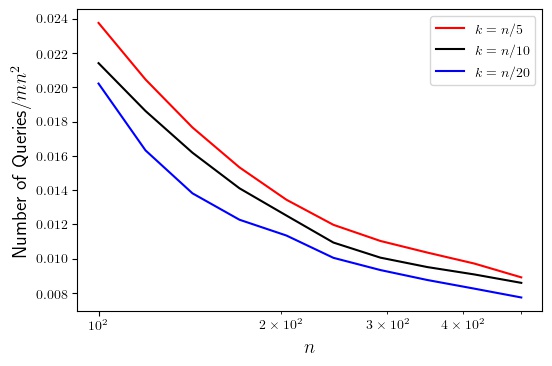}
\caption{ Varying $n$: Sampling model 2 }
\label{fig: N_variation_2}
\end{figure}

\end{document}